\documentclass[letterpaper, 10 pt, conference]{ieeeconf}  

\IEEEoverridecommandlockouts                              

\usepackage{amsmath} 
\usepackage{amssymb}  

\usepackage{float}
\usepackage{graphicx}

\makeatletter

\floatstyle{ruled}
\newfloat{algorithm}{tbp}{loa}
\providecommand{\algorithmname}{Algorithm}
\floatname{algorithm}{\protect\algorithmname}

\newtheorem{thm}{\protect\theoremname}
\newtheorem{lem}[thm]{\protect\lemmaname}
\newtheorem{rem}[thm]{\protect\remarkname}

\usepackage{balance}
\usepackage{algorithm,algpseudocode}
\algnewcommand{\Inputs}[1]{%
  \State \textbf{Inputs:}
  \Statex \hspace*{\algorithmicindent}\parbox[t]{.8\linewidth}{\raggedright #1}
}
\algnewcommand{\Initialize}[1]{%
  \State \textbf{Initialize:}
  \Statex \hspace*{\algorithmicindent}\parbox[t]{.8\linewidth}{\raggedright #1}
}

\algnewcommand{\Returns}[1]{%
  \State \textbf{Return:}
  \Statex \hspace*{\algorithmicindent}\parbox[t]{.8\linewidth}{\raggedright #1}
}

\@ifundefined{showcaptionsetup}{}{%
 \PassOptionsToPackage{caption=false}{subfig}}
\usepackage{subfig}
\makeatother

\providecommand{\lemmaname}{Lemma}
\providecommand{\remarkname}{Remark}
\providecommand{\theoremname}{Theorem}

\title{\LARGE \bf
A Hamilton\textendash Jacobi Formulation for Optimal Coordination
of Heterogeneous Multiple Vehicle Systems*
}

\author{Matthew R. Kirchner$^{1}$, Mark J. Debord$^{2}$, and Jo\~{a}o P. Hespanha$^{1}$
\thanks{*This research was supported in part by the Office of Naval Research under Grant N00014-19-WX00546 and Grant N00014-20-WX00303 and in part by the CogDeCon program under contract FA8750-18-C-0014 and contract 88ABW-2019-4739.}
\thanks{$^{1}$M. Kirchner and J. P. Hespanha  are with Center for Control, Dynamical Systems and Computation (CCDC), University of California, Santa Barbara, CA, USA. {\tt\small \{kirchner, hespanha\}@ece.ucsb.edu}}%
\thanks{$^{2}$M. Debord is with the Image and Signal Processing Branch, Research and Intelligence Department, Code D5J1000, Naval Air Warfare Center Weapons Division, China Lake, CA, USA. {\tt\small  mark.debord@navy.mil}}%
}

\begin{document}

\maketitle
\thispagestyle{empty}
\pagestyle{empty}

\begin{abstract}

We present a method for optimal coordination of multiple vehicle teams
when multiple endpoint configurations are equally desirable, such
as seen in the autonomous assembly of formation flight. The individual
vehicles' positions in the formation are not assigned a priori and
a key challenge is to find the optimal configuration assignment along
with the optimal control and trajectory. Commonly, assignment and
trajectory planning problems are solved separately. We introduce a
new multi-vehicle coordination paradigm, where the optimal goal assignment
and optimal vehicle trajectories are found simultaneously from a viscosity
solution of a single Hamilton\textendash Jacobi (HJ) partial differential
equation (PDE), which provides a necessary and sufficient condition
for global optimality. Intrinsic in this approach is that individual
vehicle dynamic models need not be the same, and therefore can be
applied to heterogeneous systems. Numerical methods to solve the HJ
equation have historically relied on a discrete grid of the solution
space and exhibits exponential scaling with system dimension, preventing
their applicability to multiple vehicle systems. By utilizing a generalization
of the Hopf formula, we avoid the use of grids and present a method
that exhibits polynomial scaling in the number of vehicles. 

\end{abstract}

\section{Introduction}

Multi-robot motion planning presents several challenges, and key among
them are trajectory planning and formation control, as well as the
assignment of vehicles to goal states. When the assignment of vehicles
to goal states is made beforehand, many methods have been proposed.
These include control theory approaches \cite{stipanovic2004decentralized,dong2015time},
graph-based techniques \cite{felner2017search}, and optimization
methods \cite{mellinger2012mixed,augugliaro2012generation}. Approximations
to assist in dimensionality reductions have also been proposed, such
as sequential methods \cite{chen2015safe} and hybrid approaches \cite{debord2018trajectory,robinson2018efficient}.

\begin{figure}
\begin{centering}
\includegraphics[width=8cm]{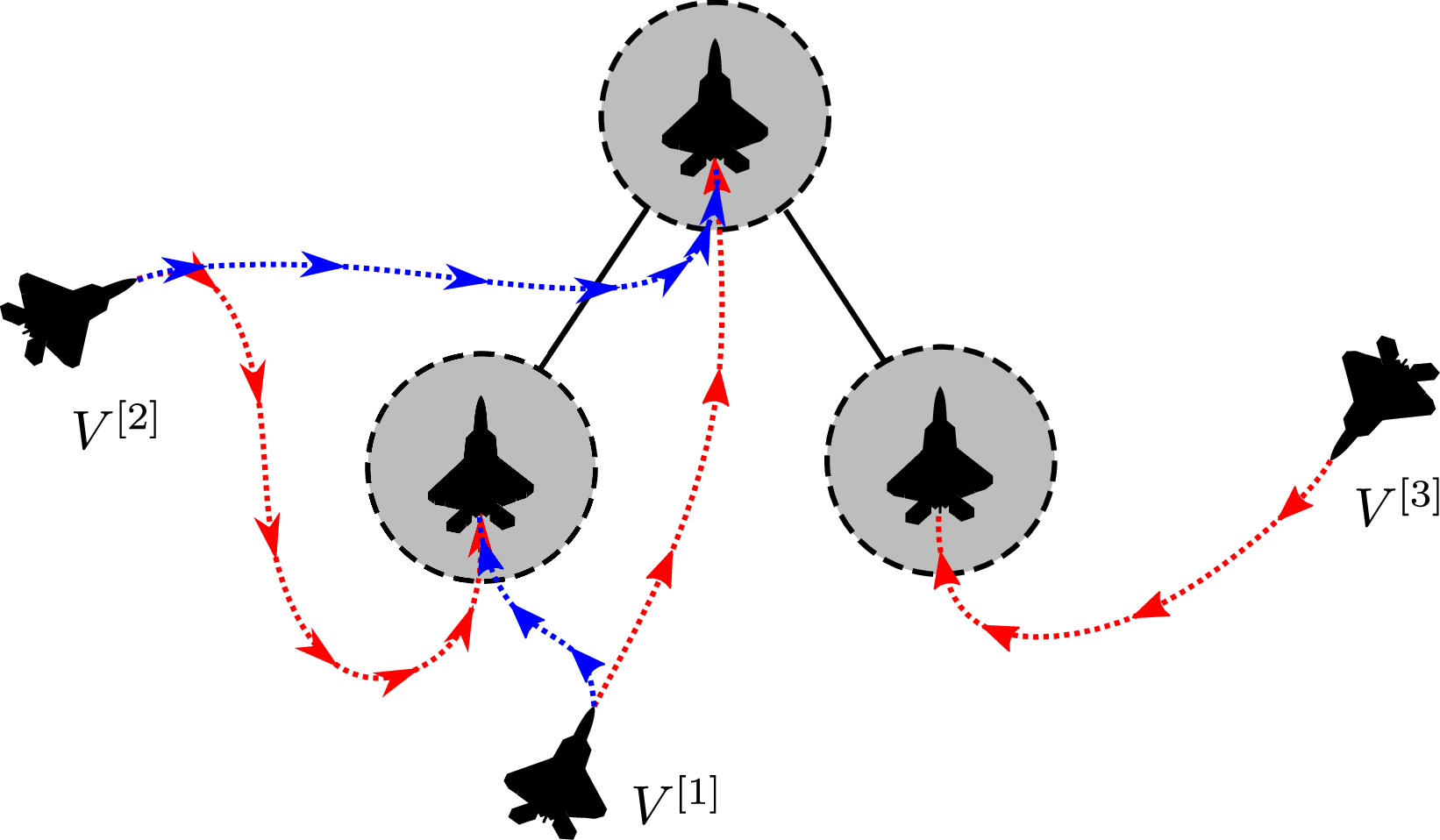}
\par\end{centering}
\caption{An illustration of the vehicle coordination problem. If a vehicle's
position in the formation is not assigned a priori, then many different
configurations are possible to achieve the desired formation, each
with a differing system objective cost. Two such possibilities are
shown, one in red and an alternate assignment in blue.\label{fig:Multi Vehicle illistration}}
\end{figure}

When the goal states of multi-vehicle teams are not set a priori,
it is necessary to determine which vehicle to allocate to which terminal
goal. There is a large body of literature developing allocation and
assignment algorithms \cite{bertsekas1991linear}, and applying these
methods to engineering problems has appeared in diverse forms including
sensor coverage \cite{marden2008distributed}, weapon-target assignment
\cite{arslan2007autonomous}, or network routing \cite{marden2012price}.
Dynamics are not considered in this body of work, which commonly study
worst-case system performance under equilibrium conditions using a
game-theoretical framework \cite{roughgarden2005selfish,marden2014generalized,kirchner_sensor_2013}.
Without consideration of dynamics, these methods have limited applicability
to vehicle coordination problems and there exists a technical gap
relating vehicle-goal allocation and motion planning, as most prior
work assumes each is done independently. The critical dependency between
assignment and trajectory planning is illustrated in Figure \ref{fig:Multi Vehicle illistration}.

There have been recent attempts to close this gap as in \cite{honig2018conflict}.
But this work only evaluates discrete spatial waypoints on a rectangular
grid without considering vehicle dynamics, and therefore the paths
generated are not feasible for most physical systems. Turpin et al. \cite{turpin2014capt} considered continuous dynamics, but was restricted to only single integrator dynamics and assumes homogeneous teams of vehicles. Morgan et al.
\cite{morgan2016swarm} generalizes this concept to linear, discrete-time
dynamics and constructs a non-linear programming problem that is solved
with sequential convex programming in conjunction with an auction
algorithm \cite{bertsekas1991linear}. However, both \cite{morgan2016swarm}
and \cite{honig2018conflict} restrict the cost function to one of
additive individual vehicle weights (similar to the cost function
presented in \cite{bertsekas1991linear}), which excludes many vehicle
optimization problems of interest, including assembling the desired
formation shape in minimal time. 

We present how the optimal assignment and trajectory can be found
simultaneously from the viscosity solution to a single Hamilton\textendash Jacobi
PDE, a necessary and sufficient condition for optimality. We assume
the vehicles can each have unique linear dynamics. Remarkably, we
show that the global viscosity solution can be found as the solution
to a linear bottleneck assignment problem (LBAP) \cite{gross1959bottleneck}.\emph{
}This fact can be utilized to construct a level set method that has
polynomial computational scaling with the number of vehicles and still
ensures a global optimum.

Our formulation has a close relation to reachability analysis which
can be useful for the design and analysis of heterogeneous systems.
When the dynamics of some vehicles differ greatly from the rest, some
formations may not be feasible (e.g. a slow-moving vehicle attempting
to join a formation with a much faster vehicle). The zero sub-level
set of the viscosity solution of a related HJ PDE defines an implicit
surface representation of the backwards reachable set \cite{mitchell2005time},
and from this one determines whether the formation can be achieved
given the unique collection of vehicle dynamics.

Traditionally, numerical solutions to HJ equations require a dense,
discrete grid of the solution space \cite{osher2006level,mitchell2008flexible}.
Computing the elements of this grid scales poorly with dimension and
has limited use for problems with dimension greater than four. The
exponential dimensional scaling in optimization is sometimes referred
to as the ``curse of dimensionality'' \cite{bellman2015adaptive}.
A new result in \cite{darbon2016algorithms} discovered a numerical
solution based on the Hopf formula \cite{hopf1965generalized} that
does not require a grid and can be used to efficiently compute solutions
of a certain class of Hamilton\textendash Jacobi PDEs. However, that
only applies to systems with time-independent Hamiltonians of the
form $\dot{x}=f\left(u\left(t\right)\right)$, and has limited use
for general linear control problems. Recently, the classes of systems
were expanded upon and generalizations of the Hopf formula are used
to solve optimal linear control problems in high-dimensions \cite{kirchner2018primaldual}
and differential games as applied to multi-vehicle, collaborative
pursuit-evasion problems \cite{kirchner2017time}.

The key contributions of this paper are to formulate the multi-vehicle
coordination problem as a single Hamilton\textendash Jacobi equation,
show that the solution of which can be found from an LBAP, and use
the generalized Hopf formula to simultaneously solve for optimal trajectory,
vehicle control, and goal assignment. By utilizing the generalized
Hopf formula, we create a level-set method to find the viscosity solution
of the HJ in a computationally efficient manner.

We introduce a multi-vehicle dynamical system and formulate the coordination
problem in Section \ref{sec:Problem-Formulation}. Section \ref{sec:Hamilton=002013Jacobi-Equations-For}
formulates the optimal vehicle coordination problem as the viscosity
solution to a single Hamilton\textendash Jacobi PDE. Section \ref{sec:A-Level-Set}
constructs a level-set method for the multi-vehicle coordination problem
and utilizes the generalized Hopf formula as a computational tool
to find the viscosity solution to the HJ PDE of the joint, mutli-vehicle
system. Section \ref{sec:Results} shows the method as applied to
several examples including a case of planar motion planning with 4
vehicles. This example shows how assignment and path planning cannot
be decoupled and even if a single vehicle has a different initial
condition, it may result in a different assignment for optimal coordination.

\section{\label{sec:Problem-Formulation}Problem Formulation}

We consider a system that consists of $N$ vehicles and each vehicle,
$i\in\mathcal{V}=\left\{ 1,\ldots,N\right\} $, has linear dynamics
\begin{equation}
\frac{d}{ds}x_{i}\left(s\right):=A_{i}x_{i}\left(s\right)+B_{i}\alpha_{i}\left(s\right),\label{eq:general dynamics}
\end{equation}
for $s\in\left[0,t\right]$, where $x_{i}\in\mathbb{R}^{n_{i}}$ is
the system state and $\alpha_{i}\left(s\right)\in\mathcal{A}_{i}\subset\mathbb{R}^{m_{i}}$
is the control input, constrained to a convex admissible control set
$\mathcal{A}_{i}$. We add the assumption that no eigenvalue of matrix
$A_{i}$ has a strictly positive real part. We let $\left[0,t\right]\ni s\mapsto\gamma_{i}\left(s;x_{i},\alpha_{i}\left(\cdot\right)\right)\in\mathbb{R}^{n_{i}}$
denote a state trajectory for vehicle $i$ that evolves in time with
measurable control sequence $\alpha_{i}\left(\cdot\right)\in\mathcal{A}_{i}$,
according to $\left(\ref{eq:general dynamics}\right)$ starting from
initial state $x_{i}$ at $s=0$. The trajectory $\gamma_{i}$ is
a solution of $\left(\ref{eq:general dynamics}\right)$ in that it
satisfies $\left(\ref{eq:general dynamics}\right)$ almost everywhere:
\begin{align*}
\frac{d}{ds}\gamma_{i}\left(s;x_{i},\alpha_{i}\left(\cdot\right)\right) & =A_{i}\gamma_{i}\left(s;x_{i},\alpha_{i}\left(\cdot\right)\right)+B_{i}\alpha_{i}\left(s\right),\\
\gamma_{i}\left(0;x_{i},\alpha_{i}\left(\cdot\right)\right) & =x_{i}.
\end{align*}

\subsection{Multi-Vehicle Model}

For the set of $N$ vehicles, we construct a joint state space with
state vector $x=\left(x_{1},\ldots,x_{i},\ldots,x_{N}\right)\in\mathbb{R}^{n}$
and control $\alpha=\left(\alpha_{1},\ldots,\alpha_{i},\ldots,\alpha_{N}\right)\in\mathcal{A}=\mathcal{A}_{1}\times\cdots\times\mathcal{A}_{i}\times\ldots\times\mathcal{A}_{N}\subset\mathbb{R}^{m}$
which is written as follows
\begin{align}
\dot{x}=\left[\begin{array}{c}
\dot{x_{1}}\\
\vdots\\
\dot{x_{i}}\\
\vdots\\
\dot{x}_{N}
\end{array}\right] & =\left[\begin{array}{ccccc}
A_{1} &  & \cdots &  & 0\\
 & \ddots\\
\vdots &  & A_{i} &  & \vdots\\
 &  &  & \ddots\\
0 &  & \cdots &  & A_{N}
\end{array}\right]\left[\begin{array}{c}
x_{1}\\
\vdots\\
x_{i}\\
\vdots\\
x_{N}
\end{array}\right]\nonumber \\
 & +\left[\begin{array}{ccccc}
B_{1} &  & \cdots &  & 0\\
 & \ddots\\
\vdots &  & B_{i} &  & \vdots\\
 &  &  & \ddots\\
0 &  & \cdots &  & B_{N}
\end{array}\right]\left[\begin{array}{c}
\alpha_{1}\\
\vdots\\
\alpha_{i}\\
\vdots\\
\alpha_{N}
\end{array}\right]\nonumber \\
\implies\dot{x} & =Ax+B\alpha,\label{eq:joint state equation}
\end{align}
We reiterate the above definition of the joint state space with the
fact that when quantities $x$, $A$, $B$, $\gamma$, and $\alpha$
appear without subscripts, they refer to the joint system in $\left(\ref{eq:joint state equation}\right)$,
and when subscripts are used, they refer to a specific individual
vehicle as defined in $\left(\ref{eq:general dynamics}\right)$.

\subsection{Vehicle Coordination}

We assume there exists a set of $N$ goals, and for each vehicle $i\in\mathcal{V}$
we associate closed convex sets $\Omega_{i,j}\subset\mathbb{R}^{n_{i}},\,j\in\mathcal{V}$
with the understanding that $x_{i}\in\Omega_{i,j}$ means the vehicle
$i$ is assigned to goal $j$. Our goal is to make sure that we have
one vehicle at each goal, but it does not matter which vehicle is
at each goal. This goal can be expressed by the requirements that
the multi-vehicle state, $x$, belongs to the following goal set
\begin{equation}
\Theta:=\left\{ x\in\mathbb{R}^{n}\,:\,\exists\sigma\in\mathcal{S}_{N}\,s.t.\,\forall i\,\in\mathcal{V},\,x_{i}\in\Omega_{i,\sigma\left(i\right)}\right\} ,\label{eq:global goal set}
\end{equation}
where $\mathcal{S}_{N}$ denotes the set of all permutations of $\mathcal{V}$.
We represent $\Theta$ implicitly with the function $J:\mathbb{R}^{n}\rightarrow\mathbb{R}$
such that
\begin{equation}
\Theta=\left\{ x\in\mathbb{R}^{n}|J\left(x\right)\leq0\right\} ,\label{eq: implicit surface rep-1}
\end{equation}
and use it to construct a cost functional for the system trajectory
$\gamma\left(s;x,\alpha\left(\cdot\right)\right)$, given terminal
time $t$ as 
\begin{equation}
R\left(t,x,\alpha\left(\cdot\right)\right)=\int_{0}^{t}C\left(\alpha\left(s\right)\right)ds+J\left(\gamma\left(t;x,\alpha\left(\cdot\right)\right)\right),\label{eq: Cost Function-1}
\end{equation}
where the function $C:\mathbb{R}^{m}\rightarrow\mathbb{R}\cup\left\{ +\infty\right\} $
represents the rate that cost is accrued over time. The value function
$\varphi:\mathbb{R}^{n}\times\left(0,+\infty\right)\rightarrow\mathbb{R}$
is defined as the minimum cost, $R$, among all admissible controls
for a given initial state $x$ with
\begin{equation}
\varphi\left(x,t\right)=\underset{\alpha\left(\cdot\right)\in\mathcal{A}}{\text{inf}}\,R\left(t,x,\alpha\left(\cdot\right)\right).\label{eq: Value function-1}
\end{equation}

\section{Hamilton\textendash Jacobi Equations For Optimal Coordination\label{sec:Hamilton=002013Jacobi-Equations-For}}

The value function in $\left(\ref{eq: Value function-1}\right)$ satisfies
the dynamic programming principle \cite{bryson1975applied,evans10}
and also satisfies the following initial value Hamilton\textendash Jacobi
(HJ) equation with $\varphi$ being the viscosity solution of
\begin{align}
\frac{\partial\varphi}{\partial s}\left(x,s\right)+H\left(s,x,\nabla_{x}\varphi\left(x,s\right)\right) & =0,\label{eq:Initial value HJ PDE-1}\\
\varphi\left(x,0\right) & =J\left(x\right),\nonumber 
\end{align}
for $s\in\left[0,t\right]$, where the Hamiltonian $H:\left(0,+\infty\right)\times\mathbb{R}^{n}\times\mathbb{R}^{n}\rightarrow\mathbb{R}\cup\left\{ +\infty\right\} $
is defined by
\begin{equation}
H\left(s,x,p\right)=-x^{\top}A^{\top}p+\underset{\alpha\in\mathbb{R}^{m}}{\text{sup}}\left\{ \left\langle -B\alpha,p\right\rangle -C\left(\alpha\right)\right\} ,\label{eq: Basic Hamiltonian definition-1}
\end{equation}
where $p:=\nabla_{x}\varphi\left(x,s\right)$ in\emph{ $\left(\ref{eq: Basic Hamiltonian definition-1}\right)$
}denotes the \emph{costate}. We denote by $\left[0,t\right]\ni s\mapsto\lambda\left(s;x,\alpha\left(\cdot\right)\right)\in\mathbb{R}^{n}$
the costate trajectory that can be shown to satisfy almost everywhere:
\begin{align*}
\frac{d}{ds}\lambda\left(s;x,\alpha\left(\cdot\right)\right) & =\nabla_{x}f\left(\gamma\left(s;x,\alpha\left(\cdot\right)\right),s\right)^{\top}\lambda\left(s;x,\alpha\left(\cdot\right)\right)\\
\lambda\left(t;x,\alpha\left(\cdot\right)\right) & =\nabla_{x}\varphi\left(\gamma\left(t;x,\alpha\left(\cdot\right)\right)\right),
\end{align*}
with initial costate denoted by $\lambda\left(0;x,\alpha\left(\cdot\right)\right)=p_{0}$.
With a slight abuse of notation, we will hereafter use $\lambda\left(s\right)$
to denote $\lambda\left(s;x,\alpha\left(\cdot\right)\right)$, when
the initial state and control sequence can be inferred through context
with the corresponding state trajectory, $\gamma\left(s;x,\alpha\left(\cdot\right)\right)$.

\subsection{System Hamiltonian}

We consider a time-optimal formulation with
\[
C\left(\alpha\right)=\mathcal{I}_{\mathcal{A}}\left(\alpha\right),
\]
where $\mathcal{I}_{\mathcal{A}}:\mathbb{R}^{m}\rightarrow\mathbb{R}\cup\left\{ +\infty\right\} $
is the characteristic function of the set of admissible controls and
is defined by
\[
\mathcal{I}_{\mathcal{A}}\left(\alpha\right)=\begin{cases}
0 & \text{if}\,\alpha\in\mathcal{A}\\
+\infty & \text{otherwise.}
\end{cases}
\]
In this case, the integral term in $\left(\ref{eq: Cost Function-1}\right)$
disappears (as long as $\alpha$ remains in $\mathcal{A})$ and $R\left(t,x,\alpha\left(\cdot\right)\right)$
is simply the value of $J\left(\gamma\left(t;x,\alpha\left(\cdot\right)\right)\right)$
at the terminal time, $t$. For this problem
\[
\varphi\left(x,t\right)=\underset{\alpha\left(\cdot\right)\in\mathcal{A}}{\text{inf}}\,J\left(\gamma\left(t;x,\alpha\left(\cdot\right)\right)\right)\leq0
\]
means that there exists a feasible trajectory for the vehicles the
ends at the state $x\left(t\right)\in\Theta$; whereas $\varphi\left(x,t\right)>0$
means that such trajectories do not exist under the system dynamics
and initial conditions. Since each vehicle has independent dynamics,
the Hamiltonian in $\left(\ref{eq: Basic Hamiltonian definition-1}\right)$
is of the form
\begin{equation}
H\left(s,x,p\right)=\sum_{i}H_{i}\left(s,x_{i},p_{i}\right),\label{eq:sum hamiltonian}
\end{equation}
where each vehicle's Hamiltonian is given by
\begin{align}
H_{i}\left(s,x_{i},p_{i}\right)= & -x_{i}^{\top}A_{i}^{\top}p_{i}\label{eq:single vehicle ham}\\
 & +\underset{\alpha_{i}\in\mathbb{R}^{m_{i}}}{\text{sup}}\left\{ \left\langle -B_{i}\alpha_{i},p_{i}\right\rangle -\mathcal{I}_{\mathcal{A}_{i}}\left(\alpha_{i}\right)\right\} .\nonumber 
\end{align}

\subsection{Linear Bottleneck Assignment}

Our goal is solving the Hamilton\textendash Jacobi PDE in $\left(\ref{eq:Initial value HJ PDE-1}\right)$,
and any $J\left(x\right)$ that satisfies $\left(\ref{eq: implicit surface rep-1}\right)$
is, in general, non-convex and presents numerical challenges. We show
that we can overcome this challenge by alternatively solving for the
global value function with the following linear bottleneck assignment
problem \cite{gross1959bottleneck}:
\begin{equation}
\varphi\left(x,t\right)=\underset{\sigma\in\mathcal{S}_{N}}{\min}\,\underset{i\in\mathcal{V}}{\max}\,\phi_{i,\sigma\left(i\right)}\left(x_{i},t\right),\label{eq:Viscosity as LBAP}
\end{equation}
where $\phi_{i,j}\left(x_{i},t\right)$ is the viscosity solution
to
\begin{align}
\frac{\partial\phi_{i,j}}{\partial s}\left(x_{i},s\right)+H_{i}\left(s,x,\nabla_{x}\phi_{i,j}\left(x_{i},s\right)\right) & =0,\label{eq: family of HJ equations}\\
\phi_{i,j}\left(x_{i},0\right) & =J_{i,j}\left(x_{i}\right),\nonumber 
\end{align}
with Hamiltonian defined by $\left(\ref{eq:single vehicle ham}\right)$.
The function $J_{i,j}:\mathbb{R}^{n_{i}}\rightarrow\mathbb{R}$ is
an implicit surface representation of $\Omega_{i,j}$ such that
\begin{equation}
\Omega_{i,j}=\left\{ x_{i}\in\mathbb{R}^{n_{i}}|J_{i,j}\left(x_{i}\right)\leq0\right\} .\label{eq:implicit of goal sets}
\end{equation}

The solution to $\left(\ref{eq:Viscosity as LBAP}\right)$ can be
found from $\phi_{i,j}$ using the appropriate linear bottleneck assignment
algorithms (for example \cite[Section 6.2]{burkhard2012assignment}),
requiring only $N^{2}$ evaluations of lower dimensional viscosity
solutions and avoiding computation involving all $\left|\mathcal{S}_{N}\right|=N!$
permutations of vehicle-goal pairs. Also of note is that if each HJ
equation $\left(\ref{eq: family of HJ equations}\right)$ has convex
initial data, it enables the use of computationally efficient techniques
that can guarantee convergence to the appropriate viscosity solution.
Additionally, since each of the solutions are independent, the $N^{2}$
solutions can be computed in parallel.

We introduce a set of mild regularity assumptions that guarantee Hamilton\textendash Jacobi
equations have a unique viscosity solution \cite[Chapter 7, p. 63]{subbotin1995generalized}:
\begin{enumerate}
\item Each Hamiltonian
\[
\left[0,t\right]\times\mathbb{R}^{n_{i}}\times\mathbb{R}^{n_{i}}\ni\left(s,x_{i},p_{i}\right)\mapsto H_{i}\left(s,x_{i},p_{i}\right)\in\mathbb{R},\,\forall i
\]
is continuous.
\item There exists a constant $c_{i}>0$ such that for all $\left(s,x_{i}\right)\in\left[0,t\right]\times\mathbb{R}^{n_{i}}$
and for all $p',p''\in\mathbb{R}^{n_{i}}$, the following inequalities
hold
\[
\left|H_{i}\left(s,x_{i},p'\right)-H_{i}\left(s,x_{i},p''\right)\right|\leq\rho_{i}\left(x_{i}\right)\left\Vert p'-p''\right\Vert ,\,\forall i
\]
and
\[
\left|H_{i}\left(s,x_{i},0\right)\right|\leq\rho_{i}\left(x_{i}\right),\,\forall i
\]
with $\rho_{i}\left(x_{i}\right)=c_{i}\left(1+\left\Vert x_{i}\right\Vert \right)$.
\item For any compact set $M\subset\mathbb{R}^{n_{i}}$ there exists a constant
$\kappa_{i}\left(M\right)>0$ such that for all $x',x''\in M$ and
for all $\left(s,p_{i}\right)\in\left[0,t\right]\times\mathbb{R}^{n_{i}}$
the inequality holds, $\forall i$
\[
\left|H_{i}\left(s,x',p_{i}\right)-H_{i}\left(s,x'',p_{i}\right)\right|\leq\mu_{i}\left(p_{i}\right)\left\Vert x'-x''\right\Vert ,
\]
with $\mu_{i}\left(p_{i}\right)=\kappa_{i}\left(M\right)\left(1+\left\Vert p_{i}\right\Vert \right)$.
\item Each terminal cost function
\[
\mathbb{R}^{n_{i}}\ni x_{i}\mapsto J_{i,j}\left(x_{i}\right)\in\mathbb{R},\,\forall i,j
\]
is continuous.
\end{enumerate}
\begin{lem}
\label{cor:global ham meets assumptions}If each vehicle Hamiltonian,
$H_{i}\left(s,x_{i},p_{i}\right)$, meets assumptions 1-3, then the
Hamiltonian defined in $\left(\ref{eq:sum hamiltonian}\right)$ also
meets assumptions 1-3.
\end{lem}
The proof is given in the appendix.
\begin{thm}
Under assumptions 1-4, $\phi_{i,j}$ is a unique viscosity solution
to $\left(\ref{eq: family of HJ equations}\right)$ for all $i,j$,
and there exists a $J:\mathbb{R}^{n}\rightarrow\mathbb{R}$ that satisfies
$\left(\ref{eq: implicit surface rep-1}\right)$ such that with the
Hamiltonian given by $\left(\ref{eq:sum hamiltonian}\right)$, $\left(\ref{eq:Viscosity as LBAP}\right)$
is a viscosity solution to $\left(\ref{eq:Initial value HJ PDE-1}\right)$.
\end{thm}
\begin{proof}
We will prove the theorem constructively by proposing a particular
$J:\mathbb{R}^{n}\rightarrow\mathbb{R}$ given as
\begin{equation}
J\left(x\right)=\underset{\sigma\in\mathcal{S}_{N}}{\min}J_{\sigma}\left(x\right),\label{eq: implicit surface J as min over J sigma}
\end{equation}
with
\begin{equation}
J_{\sigma}\left(x\right):=\underset{i\in\mathcal{V}}{\max}\,J_{i,\sigma\left(i\right)}\left(x_{i}\right),\label{eq: J sigma as max over J theta}
\end{equation}
where $J_{i,\sigma\left(i\right)}\left(x_{i}\right)$ is the implicit
representation of $\Omega_{i,j}$, defined in $\left(\ref{eq:implicit of goal sets}\right)$,
with $i,j=i,\sigma\left(i\right)$. We see that for any $\sigma\in\mathcal{S}_{N}$
when $x_{i}\in\Omega_{i,\sigma\left(i\right)},\,\forall i\in\mathcal{V}\implies J_{i,\sigma\left(i\right)}\left(x_{i}\right)\leq0,\,\forall i$
which implies $J_{\sigma}\left(x\right)\leq0$. We also see if there
exists an $i$ such that $x_{i}\notin\Omega_{i,\sigma\left(i\right)}$,
then $J_{i,\sigma\left(i\right)}>0$ and implies $J_{\sigma}\left(x\right)>0$.
Therefore the $J$ proposed in $\left(\ref{eq: implicit surface J as min over J sigma}\right)$
satisfies $\left(\ref{eq: implicit surface rep-1}\right)$ and is
an implicit surface representation of the set $\Theta$. Furthermore,
since by assumption each $J_{i,j}\left(x_{i}\right)$ is continuous,
then $\left(\ref{eq: J sigma as max over J theta}\right)$ is convex
as the max of a finite number of continuous functions is also continuous.
This implies $\left(\ref{eq: implicit surface J as min over J sigma}\right)$
is continuous as the minimum of a finite number of coninuous functions
is continuous. Therefore $\left(\ref{eq: implicit surface J as min over J sigma}\right)$
satisfies assumption 4.

From Lemma \ref{cor:global ham meets assumptions}, the system Hamiltonian
defined in $\left(\ref{eq:sum hamiltonian}\right)$ meets assumptions
1-3 and implies that $\left(\ref{eq:Initial value HJ PDE-1}\right)$
has a unique viscosity solution, denoted as $\varphi_{\text{HJ}}\left(x,t\right)$,
when the initial cost function is given by $\left(\ref{eq: implicit surface J as min over J sigma}\right)$
\cite[Theorem 8.1, p. 70]{subbotin1995generalized}. The uniqueness
of the solution $\varphi_{\text{HJ}}\left(x,t\right)$ implies that
the viscosity solution is equivalent to the value function in $\left(\ref{eq: Value function-1}\right)$.
It follows that
\begin{equation}
\varphi_{\text{HJ}}\left(x,t\right)=\underset{\alpha\left(\cdot\right)\in\mathcal{A}}{\text{inf}}\,J\left(\gamma\left(t;x,\alpha\left(\cdot\right)\right)\right),\label{eq:equiv value func}
\end{equation}
since $C\left(\alpha\right)=0$ when $\alpha\in\mathcal{A}$. Denoting
by $\alpha^{*}$ as the control that optimizes $\left(\ref{eq:equiv value func}\right)$,
we have
\[
\varphi_{\text{HJ}}\left(x,t\right)=J\left(\gamma\left(t;x,\alpha^{*}\left(\cdot\right)\right)\right).
\]
Substituting $\left(\ref{eq: J sigma as max over J theta}\right)$
and $\left(\ref{eq: implicit surface J as min over J sigma}\right)$
we have
\begin{align}
\varphi_{\text{HJ}}\left(x,t\right) & =\underset{\sigma\in\mathcal{S}_{N}}{\min}\,\underset{i\in\mathcal{V}}{\max}\,J_{i,\sigma\left(i\right)}\left(\left[\gamma\left(t;x,\alpha^{*}\left(\cdot\right)\right)\right]_{i}\right),\label{eq: value function as component}
\end{align}
where we use the notation $\left[\gamma\left(t;x,\alpha^{*}\left(\cdot\right)\right)\right]_{i}$
to represent the $i$-th block of the vector $\gamma\left(t;x,\alpha^{*}\left(\cdot\right)\right)$.
Likewise $\left(\ref{eq: family of HJ equations}\right)$ has a unique
viscosity solution,$\phi_{i,j}$, for each $i,j$ with initial cost
given $\left(\ref{eq:implicit of goal sets}\right)$, and as such
\begin{align}
\phi_{i,\sigma\left(i\right)}\left(x_{i},t\right) & =\underset{\alpha_{i}\left(\cdot\right)\in\mathcal{A}_{i}}{\text{inf}}\,J_{i,\sigma\left(i\right)}\left(\gamma_{i}\left(t;x_{i},\alpha_{i}\left(\cdot\right)\right)\right)\nonumber \\
 & =J_{i,\sigma\left(i\right)}\left(\gamma_{i}\left(t;x_{i},\alpha_{i}^{*}\left(\cdot\right)\right)\right),\label{eq: Optimal traj for vehicle-1}
\end{align}
for each $i$ and with $\alpha_{i}^{*}$ denoting the control that
optimizes $\left(\ref{eq: Optimal traj for vehicle-1}\right)$. Note
that if $J_{i,\sigma\left(i\right)}\left(\left[\gamma\left(t;x,\alpha^{*}\left(\cdot\right)\right)\right]_{i}\right)<J_{i,\sigma\left(i\right)}\left(\gamma_{i}\left(t;x_{i},\alpha_{i}^{*}\left(\cdot\right)\right)\right)$
it would contradict $\left(\ref{eq: Optimal traj for vehicle-1}\right)$
that $\alpha_{i}^{*}$ is the optimal control. Also if $J_{i,\sigma\left(i\right)}\left(\left[\gamma\left(t;x,\alpha^{*}\left(\cdot\right)\right)\right]_{i}\right)>J_{i,\sigma\left(i\right)}\left(\gamma_{i}\left(t;x_{i},\alpha_{i}^{*}\left(\cdot\right)\right)\right)$,
then it would contradict $\left(\ref{eq:equiv value func}\right)$
that $\alpha^{*}$ is the optimal control of the entire system. Therefore,
$J_{i,\sigma\left(i\right)}\left(\left[\gamma\left(t;x,\alpha^{*}\left(\cdot\right)\right)\right]_{i}\right)=J_{i,\sigma\left(i\right)}\left(\gamma_{i}\left(t;x_{i},\alpha_{i}^{*}\left(\cdot\right)\right)\right)=\phi_{i,\sigma\left(i\right)}\left(x_{i},t\right)$
and our value function in $\left(\ref{eq:equiv value func}\right)$
becomes
\[
\varphi_{\text{HJ}}\left(x,t\right)=\underset{\sigma\in\mathcal{S}_{N}}{\min}\,\underset{i\in\mathcal{V}}{\max}\,\phi_{i,\sigma\left(i\right)}\left(x_{i},t\right),
\]
and we arrive at our result.
\end{proof}

\section{\label{sec:A-Level-Set}A Level Set Method with the Generalized Hopf
Formula}

First, we introduce the Fenchel\textendash Legendre transform of a
convex, proper, lower semicontinuous function $g:\mathbb{R}^{\ell}\rightarrow\mathbb{R}\cup\left\{ +\infty\right\} $,
denoted as $g^{\star}:\mathbb{R}^{\ell}\rightarrow\mathbb{R}\cup\left\{ +\infty\right\} $,
and is defined as \cite{hiriart2012fundamentals}
\begin{equation}
g^{\star}\left(p\right)=\underset{y\in\mathbb{R}^{\ell}}{\text{sup}}\left\{ \left\langle p,y\right\rangle -g\left(y\right)\right\} .\label{eq: Fenchel transform}
\end{equation}
We propose to find the viscosity solutions of $\left(\ref{eq: family of HJ equations}\right)$
using the generalized Hopf formula.
\begin{thm}
Each viscosity solution of $\left(\ref{eq: family of HJ equations}\right)$
can be found from the formula
\begin{align}
\phi_{i,j}\left(x_{i},t\right)= & -\underset{p_{i}\in\mathbb{R}^{n_{i}}}{\text{min}}\Bigg\{ J_{i,j}^{\star}\left(e^{-tA_{i}^{\top}}p_{i}\right),\label{eq:gen hopf formula}\\
 & +\int_{0}^{t}\widehat{H_{i}}\left(s,p_{i}\right)ds-\left\langle x_{i},p_{i}\right\rangle \Bigg\}\nonumber 
\end{align}
with
\begin{equation}
\widehat{H}_{i}\left(s,p_{i}\right)=\underset{\alpha\in\mathcal{A}_{i}}{\text{sup}}\left\{ \left\langle -e^{-sA_{i}}B_{i}\alpha_{i},p_{i}\right\rangle \right\} .\label{eq:transformed ham}
\end{equation}
\end{thm}
\begin{proof}
Proceeding similar to \cite{kirchner2017time}, we apply a change
of variables to $\left(\ref{eq:general dynamics}\right)$ with 
\begin{equation}
z_{i}\left(s\right)=e^{-sA_{i}}x_{i}\left(s\right),\label{eq:change of varibles}
\end{equation}
which results in the following system
\begin{equation}
\frac{d}{ds}z_{i}\left(s\right)=e^{-sA_{i}}B_{i}\alpha_{i}\left(s\right).\label{eq:z transformed system}
\end{equation}
By utilizing this change of variables, we construct an equivalent
HJ equation

\begin{align}
\frac{\partial\tilde{\phi}_{i,j}}{\partial s}\left(z_{i},s\right)+\widehat{H}_{i}\left(s,\nabla_{z}\tilde{\phi}_{i,j}\left(z_{i},s\right)\right) & =0,\label{eq:Non-state dependent HJ}\\
\tilde{\phi}_{i,j}\left(z,0\right) & =J_{i,j}\left(e^{tA}z_{i}\right).\nonumber 
\end{align}
From Lemma \ref{lem:Hopf Ham equiv} (given in the appendix), $\left(\ref{eq:transformed ham}\right)$
meets assumption 1-3, and by composition rule $J_{i,j}\left(e^{tA}z_{i}\right)$
is continuous \cite[Theorem 4.7]{rudin1964principles} and meets assumption
4. Therefore $\left(\ref{eq:Non-state dependent HJ}\right)$ has a
unique viscosity solution that is equivalent to the cost functional
\[
\tilde{\phi}_{i,j}\left(z_{i},t\right)=\underset{\alpha_{i}\left(\cdot\right)\in\mathcal{A}_{i}}{\text{inf}}\,J_{i,j}\left(e^{tA}\xi_{i}\left(t;z_{i},\alpha_{i}\left(\cdot\right)\right)\right),
\]
where $\xi_{i}\left(s;z_{i},\alpha\left(\cdot\right)\right)$ is a
solution trajectory that satisfies $\left(\ref{eq:z transformed system}\right)$
almost everywhere. Since
\[
z_{i}\left(t\right)=e^{-tA_{i}}x_{i}\left(t\right),
\]
we have
\begin{align*}
\tilde{\phi}_{i,j}\left(z_{i},t\right) & =\underset{\alpha_{i}\left(\cdot\right)\in\mathcal{A}_{i}}{\text{inf}}\,J_{i,j}\left(e^{tA}\xi_{i}\left(t;z_{i},\alpha_{i}\left(\cdot\right)\right)\right)\\
 & =\underset{\alpha_{i}\left(\cdot\right)\in\mathcal{A}_{i}}{\text{inf}}\,J_{i,j}\left(\gamma_{i}\left(t;x_{i},\alpha_{i}\left(\cdot\right)\right)\right)\\
 & =\phi_{i,j}\left(x_{i},t\right),
\end{align*}
noting that since $x_{i}=\gamma_{i}\left(0;x_{i},\alpha_{i}\left(\cdot\right)\right)$
by the transform $\left(\ref{eq:change of varibles}\right)$ implies
$z_{i}=x_{i}$ at $t=0$. Thus, $\tilde{\phi}\left(z_{i},t\right)=\phi\left(x_{i},t\right)$
and we can find $\phi\left(x_{i},t\right)$ by finding the viscosity
solution to $\left(\ref{eq:Non-state dependent HJ}\right)$.

Since $\Omega_{i,j}$ is assumed closed and convex, this implies $J_{i,j}$
is convex and by assumption 4 is continuous in $z_{i}$. By assumption
1, $H\left(p_{i}\right)$ is continuous in $p_{i}$, therefore $\left(\ref{eq:gen hopf formula}\right)$
gives an exact, point-wise viscosity solution for $\left(\ref{eq:Non-state dependent HJ}\right)$
\cite[Section 5.3.2, p. 215]{kurzhanski2014dynamics}.
\end{proof}
Formula $\left(\ref{eq:gen hopf formula}\right)$ shows that we can
compute a viscosity solution to $\left(\ref{eq:Non-state dependent HJ}\right)$
by solving a finite dimensional optimization problem. This avoids
constructing a discrete spatial grid, and is numerically efficient
to compute even when the dimension of the state space is large. Additionally,
no spatial derivative approximations are needed with Hopf formula-based
methods, and this eliminates the numeric dissipation introduced with
the Lax\textendash Friedrichs scheme \cite{osher1991high}, which
is necessary to maintain numeric stability in grid-based methods.

\subsection{\label{subsec:numerics}Numerical Optimization of the Hopf Formula}

We transcribe the Hopf formula into a non-linear programming problem
by approximating the integral in $\left(\ref{eq:gen hopf formula}\right)$
with an $N$-point quadrature rule sampled on the time grid 
\[
\pi^{N}=\left\{ s_{k}:k=0,\ldots,N\right\} ,
\]
with $s_{k}\in\left[0,t\right]$ and corresponding weights $w_{k}$.
Additionally, we make a simple, invertible change of variable $\tilde{p}_{i}=e^{-tA_{i}^{\top}}p_{i}$
and substituting into $\left(\ref{eq:gen hopf formula}\right)$ results
in the following unconstrained optimization problem that solves $\left(\ref{eq:gen hopf formula}\right)$:
\begin{equation}
\begin{cases}
\underset{\tilde{p}_{i}}{\min} & J_{i,j}^{\star}\left(\tilde{p}_{i}\right)+\sum_{j=0}^{N}w_{j}\widehat{H}_{i}\left(s_{k},\tilde{p}_{i}\right)\end{cases}-\left\langle e^{tA_{i}}x_{i},\tilde{p}_{i}\right\rangle ,\label{eq:Hopf change of var}
\end{equation}
with $\widehat{H}_{i}$ now being defined as
\begin{equation}
\widehat{H}_{i}\left(s,\tilde{p}_{i}\right)=\underset{\alpha_{i}\in\mathcal{A}_{i}}{\text{sup}}\left\{ \left\langle -B_{i}\alpha_{i},e^{sA_{i}^{\top}}\tilde{p}_{i}\right\rangle \right\} .\label{eq:H for change of vars}
\end{equation}

The variable transformation is done so that when the matrix $A$ has
at least one eigenvalue with strictly negative real part, we avoid
computation of $e^{-sA^{\top}}$ which would have exponentially unstable
poles since we are evaluating the matrix exponential of $-A$. This
divergence would cause sensitivity in the evaluation of the Hopf formula
in $\left(\ref{eq:gen hopf formula}\right)$ with respect to small
changes of $p$. By utilizing the variable transformation and optimizing
$\left(\ref{eq:Hopf change of var}\right)$, we avoid this sensitivity.
\begin{rem}
Often the expression in $\left(\ref{eq:H for change of vars}\right)$
is known in closed form and we present one such common case. Recall
that $\left(\cdot\right)^{\star}$ denotes the Fenchel\textendash Legendre
transform of a function defined in $\left(\ref{eq: Fenchel transform}\right)$,
and suppose $\mathcal{A}_{i}$ is the closed convex set defined as
\[
\mathcal{A}_{i}:=\left\{ u:\left\Vert u\right\Vert \leq1\right\} ,
\]
where $\left\Vert \left(\cdot\right)\right\Vert $ is any norm. Then
$\left(\mathcal{I}_{\mathcal{A}_{i}}\right)^{\star}$ defines a norm,
which we denote with $\left\Vert \left(\cdot\right)\right\Vert _{*}$,
which is the dual norm to $\left\Vert \left(\cdot\right)\right\Vert $
\cite{hiriart2012fundamentals}. From this we write $\left(\ref{eq:H for change of vars}\right)$
as
\begin{equation}
\widehat{H}_{i}\left(s,\tilde{p}_{i}\right)=\left\Vert -B_{i}^{\top}e^{sA_{i}^{\top}}\tilde{p}_{i}\right\Vert _{*}.\label{eq: Hamiltonian dual norm in z-1}
\end{equation}
\end{rem}

\subsection{Time-Optimal Control to a Goal Set\label{subsec:Time-Optimal-Control-to}}

The task of determining the control that drives the system into $\Theta$
in minimal time can be determined by finding the smallest $t$ such
that
\begin{equation}
\varphi\left(x,t\right)\leq0.\label{eq:root of phi}
\end{equation}
When the system $\left(\ref{eq:general dynamics}\right)$ is constrained
controllable, then the set of times such that $x$ is reachable with
respect to $\Theta$ contains the open interval $[t',\infty)$ \cite{fashoro1992controllability}.
This insures that if $x$ is outside the set $\Theta$ at time $t=0$,
then $\varphi\left(x,0\right)>0$ and there exists a time $t'$ such
that $\varphi\left(x,\tau\right)<0$ for all $\tau>t'$. This implies
standard root-finding algorithms can be used to find $\left(\ref{eq:root of phi}\right)$.
As noted in \cite{kirchner2018primaldual}, we solve for the minimum
time to reach the set $\Theta$ by constructing a newton iteration,
starting from an initial guess, $t_{0}$, with
\begin{equation}
t_{k+1}=t_{k}-\frac{\varphi\left(x,t_{k}\right)}{\frac{\partial\varphi}{\partial t}\left(x,t_{k}\right)},\label{eq:Newton iterate}
\end{equation}
where $\varphi\left(x,t_{k}\right)$ is the solution to $\left(\ref{eq:Initial value HJ PDE-1}\right)$
at time $t_{k}$. The value function must satisfy the HJ equation
\[
\frac{\partial\varphi}{\partial t}\left(x,t_{k}\right)=-H\left(s,x,\nabla_{x}\varphi\left(x,t_{k}\right)\right),
\]
where $\nabla_{x}\varphi\left(x,t_{k}\right)=\left(e^{t_{k}A_{1}^{\top}}\tilde{p}_{1}^{*},\cdots,e^{t_{k}A_{N}^{\top}}\tilde{p}_{N}^{*}\right)$
and each $\tilde{p}^{*}$ is the argument of the minimizer in $\left(\ref{eq:Hopf change of var}\right)$.
We iterate $\left(\ref{eq:Newton iterate}\right)$ until convergence
at the optimal time to reach, which we denote as $t^{*}$. This process
is summarized in Algorithm \ref{alg:Algorithm}.

The optimal control and trajectory for each vehicle is found directly
from the necessary conditions of optimality established by Pontryagin's
principal \cite{pontryagin2018mathematical} by noting the optimal
trajectory must satisfy
\begin{align*}
\frac{d}{ds}\gamma_{i}^{*}\left(s;x_{i},\alpha_{i}^{*}\left(\cdot\right)\right) & =-\nabla_{p}H\left(s,x_{i}\left(s\right),\lambda_{i}^{*}\left(s\right)\right)\\
 & =A_{i}\gamma_{i}\left(s;x_{i},\alpha_{i}^{*}\left(\cdot\right)\right)\\
 & +B_{i}\nabla_{p}\left\Vert -B_{i}^{\top}\lambda_{i}^{*}\left(s\right)\right\Vert _{*},
\end{align*}
where $\lambda_{i}^{*}$ is the optimal costate trajectory and is
given by
\[
\lambda_{i}^{*}\left(s\right)=e^{-\left(t^{*}-s\right)A_{i}^{\top}}\tilde{p}_{i}^{*}.
\]
This implies our optimal control is
\begin{equation}
\alpha_{i}^{*}\left(s\right)=\nabla_{p}\left\Vert -B_{i}^{\top}e^{-\left(t^{*}-s\right)A_{i}^{\top}}\tilde{p}_{i}^{*}\right\Vert _{*},\label{eq: optimal control}
\end{equation}
for all time $s\in\left[0,t^{*}\right]$, provided the gradient exists.

\begin{algorithm}
\begin{algorithmic}[1]
\Inputs{$x_i,\,J_{i,j},\, \forall i,j$}
\Initialize{$t = t_0$, $\varphi=\infty$, $\epsilon=10^{-5}$}
\While{$|\varphi|\geq\epsilon$}
	\ForAll{$i,j$}
\State{$Q_{i,j}=\underset{\tilde{p}_{i}}{\min}\,J_{i,j}^{\star}\left(\tilde{p}_{i}\right)+\sum_{j=0}^{N}w_{j}\widehat{H}_{i}\left(s_{k},\tilde{p}_{i}\right)$\hskip\algorithmicindent \hspace*{33mm} $-\left\langle e^{tA_{i}}x_{i},\tilde{p}_{i}\right\rangle$}
	\EndFor
	\State{$\varphi=\text{LBAP}\left(Q\right)$}
	\State{$p=\left(e^{tA_{1}^{\top}}\tilde{p}_{1}^{*},\cdots,e^{tA_{i}^{\top}}\tilde{p}_{i}^{*},\cdots,e^{tA_{N}^{\top}}\tilde{p}_{N}^{*}\right)$}

	\State{$t=t+\frac{\varphi}{H\left(x,p\right)}$ }
\EndWhile
\Returns{$p$, $t$}
\end{algorithmic}

\caption{Algorithm to compute the viscosity solution at the minimum time to
reach of the set $\Theta$.\label{alg:Algorithm}}
\end{algorithm}

\section{\label{sec:Results}Results}

We present results for two examples. For the first we choose a toy
example of two vehicles, each with differing single integrator dynamics.
The purpose of this example is that it allows easy visualization of
the level set propagation and the solution can be verified by comparing
with methods such as \cite{mitchell2005toolbox}, which cannot be
applied to the higher dimensional second example to follow. The LBAP
was solved using \cite[Algorithm 6.1]{burkhard2012assignment} and
the optimization $\left(\ref{eq:Hopf change of var}\right)$ used
sequential quadratic programming \cite[Chapter 18]{nocedal2006numerical}
with initial conditions $\tilde{p}_{i}=e^{tA_{i}}x_{i}\left(0\right)$
for each vehicle. To prevent a singular control condition, a slight
smoothing was applied to the Hamiltonian in $\left(\ref{eq: optimal control}\right)$
and $\left(\ref{eq:H for change of vars}\right)$ using \cite{ramirez2014}
with parameter $\mu=10^{-6}$.

\subsection{\label{subsec:Toy-Problem}Toy Problem}

\begin{figure}
\begin{centering}
\includegraphics[width=8cm]{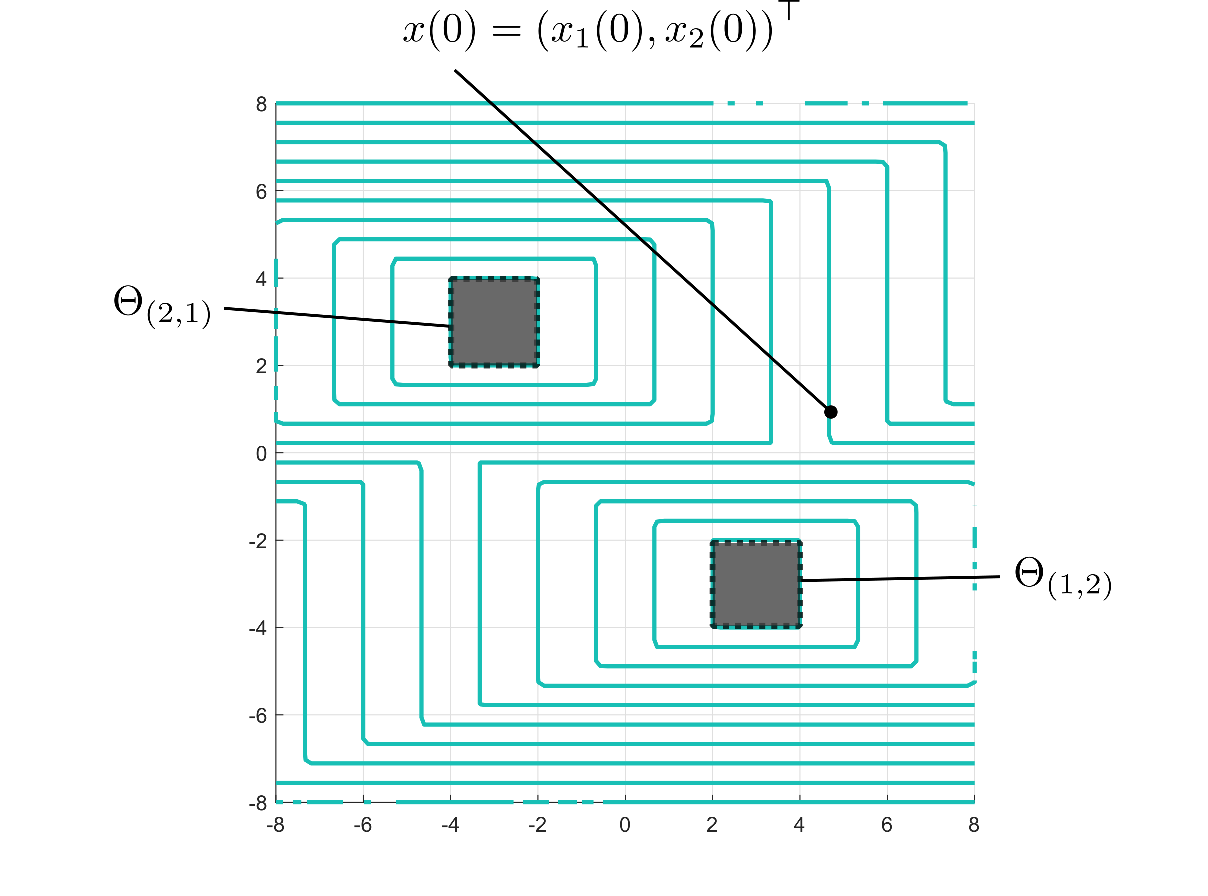}
\par\end{centering}
\caption{The two vehicle example of Section \ref{subsec:Toy-Problem}. The
zero level set evolution solved at various times shown in green. \label{fig:Example 1 level sets} }
\end{figure}
We present a two vehicle problem where the dynamics are $\dot{x}_{1}=3\alpha_{1}$
for the first vehicle and $\dot{x}_{2}=\alpha_{2}$ for the second
vehicle. The states $x_{1},x_{2}$ are the linear position of each
respective vehicle. The control is bounded by $\left|\alpha_{i}\right|\leq1$.
The goal states are $\left|x_{i}-3\right|\leq1$ for $j=1$ and $\left|x_{i}+3\right|\leq1$
for $j=2$. The level set contours of the joint space are shown in
Figure \ref{subsec:Toy-Problem}, for ten time samples equally spaced
on $t\in\left[0,4\right]$. Notice this seemingly simple heterogeneous
system can give rise to counterexamples for the assignment formulations
given in \cite{morgan2016swarm} where the sum of the vehicle distances
are used as the assignment metric and \cite{honig2018conflict} where
the time of arrival of vehicles was used. Take the initial state $x\left(0\right)=\left(4.667,0.5\right)^{\top}$,
which is shown on Figure \ref{fig:Example 1 level sets}. For the
time metric, the assignment $\left(1,2\right)$ gives $0.222+2.5=2.7223$,
while the assignment $\left(2,1\right)$ gives $2.222+1.5=3.722$,
indicating $\left(1,2\right)$ is the clear choice. However, $\left(2,1\right)$
is the global optimum in this example, as both vehicles reach their
desired goal states in a time of $2.222$, as the zero level set of
the global value function intersects the point $x\left(0\right)$
at that time, giving the minimum time-to-reach. For the distance metric,
the misassignment is more pronounced, with the assignment $\left(1,2\right)$
giving $0.667+2.5=3.1670$ as compared to the assignment $\left(2,1\right)$
gives $4.667+1.5=8.167$.

\subsection{Planar Motion\label{subsec:Planar-Motion}}

\begin{figure*}
\begin{centering}
\subfloat[\label{fig:4 veh. 10 percent}Optimal trajectories at $t=1.50$ (10\%
of path).]{\centering{}\includegraphics[width=5.75cm]{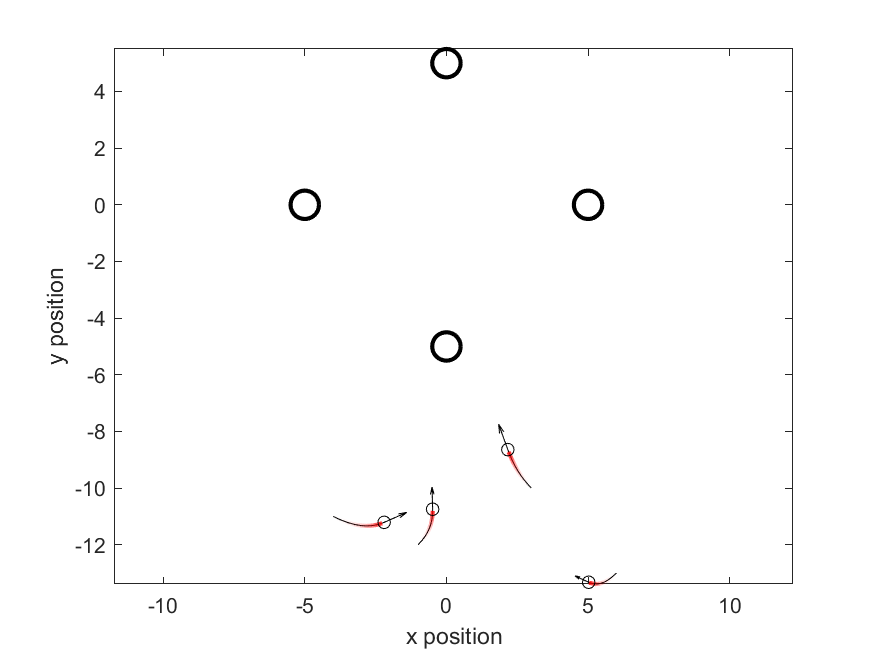}}\hspace*{\fill}\subfloat[\label{fig:4 vehicles complete}Complete optimal trajectories.]{\begin{centering}
\includegraphics[width=5.75cm]{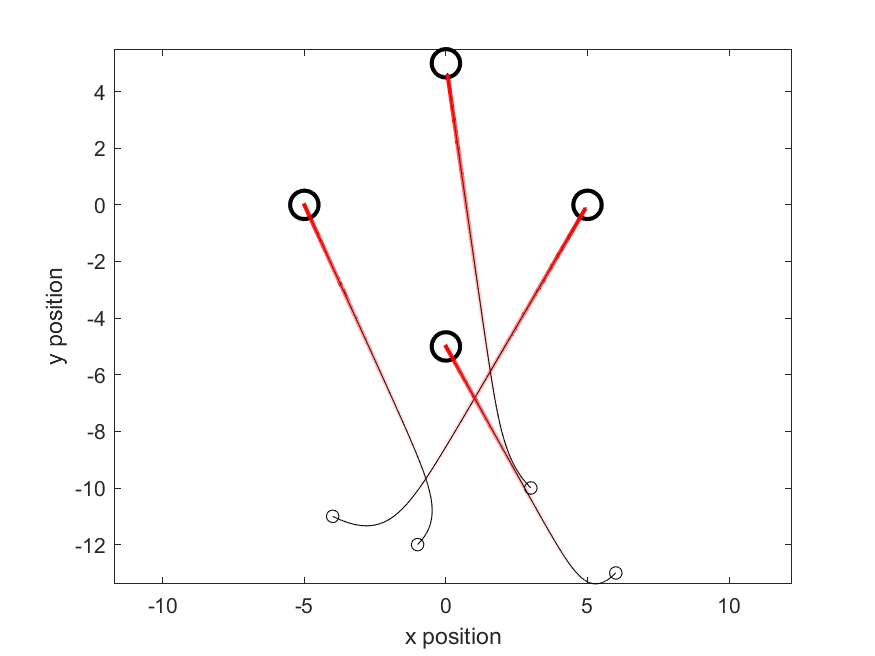}
\par\end{centering}
}\hspace*{\fill}\subfloat[\label{fig:Alternate}Complete optimal trajectories for an alternate
initial condition.]{\centering{}\includegraphics[width=5.75cm]{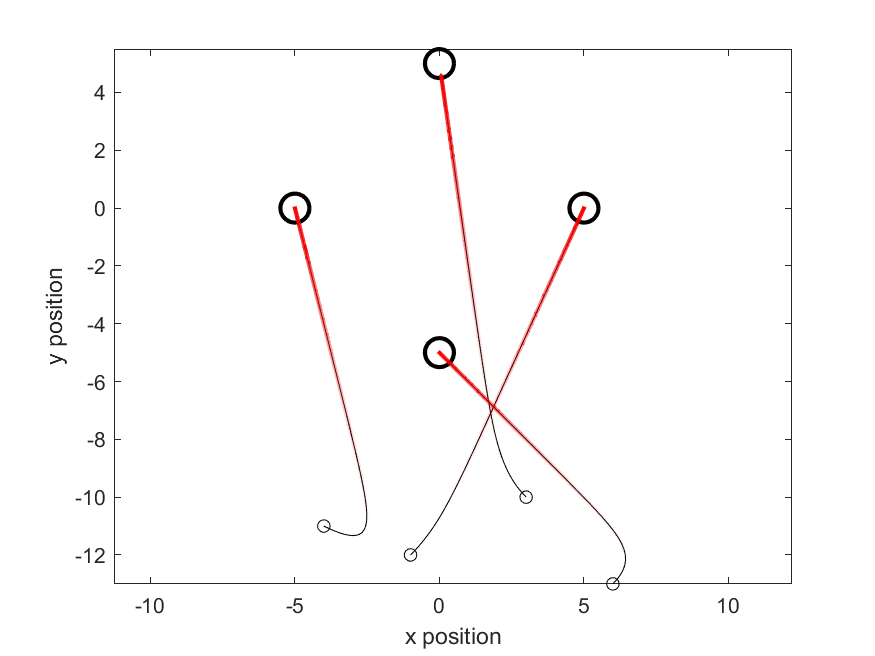}}
\par\end{centering}
\caption{The example presented in Section \ref{subsec:Planar-Motion} with
four vehicles guiding to four possible goal sets. The optimal formation
is achieved in $t^{*}=15.015$. The trajectories are shown for two
different times as they traverse the optimal path in the left two
figures. The figure on the right is the trajectories when the initial
condition for vehicle 4 is different. \label{fig:4 vehicle example}}
\end{figure*}

We choose for dynamics $\left(\ref{eq:general dynamics}\right)$ with
state $x\in\left[r,\dot{r}\right]^{\top}$, where $r\in\mathbb{R}^{2}$
is spatial position of a robot and $\dot{r}\in\mathbb{R}^{2}$ is
the velocity and 
\[
A_{i}=\left[\begin{array}{cccc}
0 & 0 & 1 & 0\\
0 & 0 & 0 & 1\\
0 & 0 & -1 & 0\\
0 & 0 & 0 & -1
\end{array}\right],\,B_{i}=\left[\begin{array}{cc}
0 & 0\\
0 & 0\\
1 & 0\\
0 & 1
\end{array}\right],
\]
for each vehicle $i\in\mathcal{V}$. The control $\alpha_{i}\in\mathbb{R}^{2}$
is constrained to lie in the set $\left\Vert \alpha_{i}\right\Vert _{2}\leq1$.
The robots are tasked with reaching the goal formation and coming
to rest, in minimum time. The goal sets each have radius of $0.5$
and the centers of the goals are located spatially at $\left(0,5\right)^{\top}$,
$\left(-5,0\right)^{\top}$, $\left(5,0\right)^{\top}$, and $\left(0,-5\right)^{\top}$.
Since the 2-norm is self-dual, the Hamiltonian $\left(\ref{eq: Hamiltonian dual norm in z-1}\right)$
for each vehicle is 
\[
\widehat{H}_{i}\left(s,\tilde{p}_{i}\right)=\left\Vert -B_{i}^{\top}e^{sA_{i}^{\top}}\tilde{p}_{i}\right\Vert _{2}.
\]
Since $A_{i}$ contains eigenvalues with negative real part, we optimize
using the variable transformation of Section \ref{subsec:numerics}.
The initial conditions of the vehicles were
\begin{align*}
x_{1}\left(0\right) & =\left(3,-10,-1,1\right)^{\top},\\
x_{2}\left(0\right) & =\left(-1,-12,1,1\right)^{\top},\\
x_{3}\left(0\right) & =\left(-4,-11,2,-1\right)^{\top},\\
x_{4}\left(0\right) & =\left(6,-13,-1,-1\right)^{\top}.
\end{align*}
Only $11$ iterations of $\left(\ref{eq:Newton iterate}\right)$ were
needed to solve for the minimum time to reach the goal formation,
which was found to be $t^{*}=15.015$. Figures \ref{fig:4 veh. 10 percent}
and \ref{fig:4 vehicles complete} show the optimal paths found from
$\left(\ref{eq: optimal control}\right)$. Figure \ref{fig:Alternate}
shows the optimal paths when the initial condition for vehicle 4 was
changed to $x_{4}\left(0\right)=\left(6,-13,1,1\right)^{\top}$ and
a different optimal assigment results.
\section{Conclusions and Future Work}

We presented how to formulate vehicle coordination problems with unknown
goal assignments as the viscosity solution to a single Hamilton\textendash Jacobi
equation. We show the solution of this single HJ PDE is equivalent
to decomposing the problem and performing a linear bottleneck assignment
using the viscosity solutions of independent single-vehicle problems.
This allows quadratic computational scaling in the number of vehicles.
Finally, a level set method based on the Hopf formula was presented
for efficient computation of the vehicle value functions, in which
each can be computed in parallel. The Hamilton\textendash Jacobi formulation
presented has other advantages for multi-robot systems, such as the
compensation of time delays which can be induced in several ways including
computation, sensing, and inter-robot communication. See for example
\cite{kirchner2019timedelay}.

Future work includes to expand the class of allowable vehicle dynamics
from general linear models to certain types of nonlinear dynamics
and to allow for dependencies between vehicles both in the dynamics
models and in the cost functional.

\section*{APPENDIX}

\begin{lem}
\label{lem:Hopf Ham equiv}Given that $\left(\ref{eq:single vehicle ham}\right)$
satisfies assumptions 1-3, then $\left(\ref{eq:transformed ham}\right)$
also satisfies assumption 1-3.
\end{lem}
\begin{proof}
From assumption 1, $\widehat{H}_{i}$ is continuous by composition
rule \cite[Theorem 4.7]{rudin1964principles} and since $\left(\ref{eq:Non-state dependent HJ}\right)$
does not depend on state, it therefore trivially meets assumption
3. It remains to show that $\left(\ref{eq:transformed ham}\right)$
meets assumption 2. We know assumption 2 holds for all $x_{i}\in\mathbb{R}^{n_{i}}$
and for all $p',p''\in\mathbb{R}^{n_{i}}$, therefore the following
holds
\begin{align*}
 & \left|H_{i}\left(s,x_{i},e^{-sA_{i}^{\top}}p^{'}\right)-H_{i}\left(s,x_{i},e^{-sA_{i}^{\top}}p''\right)\right|\\
 & \leq c_{i}\left(1+\left\Vert x_{i}\right\Vert \right)\left\Vert e^{-sA_{i}^{\top}}p'-e^{-sA_{i}^{\top}}p''\right\Vert \\
 & \leq c_{i}\left(1+\left\Vert x_{i}\right\Vert \right)\left|\lambda_{\max}\left(e^{-sA_{i}^{\top}}\right)\right|\left\Vert p'-p''\right\Vert \\
 & =c\left(1+\left\Vert x_{i}\right\Vert \right)\left\Vert p'-p''\right\Vert .
\end{align*}
When $x_{i}=0$ we have
\begin{align*}
 & \left|H_{i}\left(s,0,e^{-sA_{i}^{\top}}p^{'}\right)-H_{i}\left(s,0,e^{-sA_{i}^{\top}}p''\right)\right|\\
 & =\left|\widehat{H}_{i}\left(s,p'\right)-\widehat{H}_{i}\left(s,p''\right)\right|\\
 & \leq c\left\Vert p'-p''\right\Vert \leq c\left(1+\left\Vert x_{i}\right\Vert \right)\left\Vert p'-p''\right\Vert ,
\end{align*}
for any $x_{i}$. And since $\left|\widehat{H}_{i}\left(s,x_{i},0\right)\right|=0$
for all $x_{i}$, it follows that assumption 2 is met. 
\end{proof}
\medskip{}

Proof of Lemma \ref{cor:global ham meets assumptions}.
\begin{proof}
The sum of continuous functions is also continuous \cite[Theorem 4.9]{rudin1964principles},
therefore assumption 1 is met. To prove assumption 2, we first write,
for all $\overline{p}',\overline{p}''\in\mathbb{R}^{n}$,
\begin{align}
 & \left|H\left(s,x,\overline{p}'\right)-H\left(s,x,\overline{p}''\right)\right|\nonumber \\
 & =\left|\sum_{i}H_{i}\left(s,x_{i},\overline{p}_{i}'\right)-\sum_{i}H_{i}\left(s,x_{i},\overline{p}_{i}''\right)\right|\nonumber \\
 & \leq\sum_{i}\left|H_{i}\left(s,x_{i},\overline{p}_{i}'\right)-H_{i}\left(s,x_{i},\overline{p}_{i}''\right)\right|,\label{eq:triangular equality}\\
 & \leq\sum_{i}c_{i}\left(1+\left\Vert x_{i}\right\Vert \right)\left\Vert \overline{p}_{i}'-\overline{p}_{i}''\right\Vert \nonumber \\
 & \leq\sum_{i}c_{i}\left(1+\left\Vert x\right\Vert \right)\left\Vert \overline{p}'-\overline{p}''\right\Vert \label{eq: component wise inequality}\\
 & c\left(1+\left\Vert x\right\Vert \right)\left\Vert \overline{p}'-\overline{p}''\right\Vert ,\nonumber 
\end{align}
where line $\left(\ref{eq:triangular equality}\right)$ comes from
the triangular inequality and line $\left(\ref{eq: component wise inequality}\right)$
comes by noting $\forall i$, $\left\Vert x_{i}\right\Vert \leq\left\Vert x\right\Vert $
and $\left\Vert \overline{p}_{i}'-\overline{p}_{i}''\right\Vert \leq\left\Vert \overline{p}'-\overline{p}''\right\Vert $.
Therefore, there exists a $c=\sum_{i}c_{i}$ such that the inequality
holds and we arrive at our result for part 1 of assumption 2. The
second part of assumption follows from part 1,
\begin{align*}
 & \left|H\left(s,x,0\right)\right|=\left|\sum_{i}H\left(s,x_{i},0\right)\right|\\
 & \leq\sum_{i}\left|H\left(s,x_{i},0\right)\right|\leq\sum_{i}c_{i}\left(1+\left\Vert x_{i}\right\Vert \right)\\
 & \leq\sum_{i}c_{i}\left(1+\left\Vert x\right\Vert \right)=c\left(1+\left\Vert x\right\Vert \right).
\end{align*}
And we have shown part 2. The proof of assumption 3 follows that of
above. For any compact set $M\subset\mathbb{R}^{n}$ and for all $\overline{x}',\overline{x}''\in M$
\begin{align*}
 & \left|H\left(s,\overline{x}',p\right)-H\left(s,\overline{x}'',p\right)\right|\\
 & =\left|\sum_{i}H\left(s,\overline{x}_{i}',p_{i}\right)-\sum_{i}H_{i}\left(s,\overline{x}_{i}'',p_{i}\right)\right|\\
 & \leq\sum_{i}\left|H\left(s,\overline{x}_{i}',p_{i}\right)-H_{i}\left(s,\overline{x}_{i}'',p_{i}\right)\right|,\\
 & \leq\sum_{i}\kappa_{i}\left(M\right)\left(1+\left\Vert p_{i}\right\Vert \right)\left\Vert \overline{x}_{i}'-\overline{x}_{i}''\right\Vert \\
 & \leq\sum_{i}\kappa_{i}\left(M\right)\left(1+\left\Vert p\right\Vert \right)\left\Vert \overline{x}'-\overline{x}''\right\Vert \\
 & =\kappa\left(M\right)\left(1+\left\Vert p\right\Vert \right)\left\Vert \overline{x}'-\overline{x}''\right\Vert .
\end{align*}
Therefore, there exists a $\kappa\left(M\right)=\sum_{i}\kappa_{i}\left(M\right)$
such that the inequality holds.
\end{proof}

\balance

\bibliographystyle{IEEEtran}
\bibliography{WP}

\end{document}